\documentclass[sigconf]{aamas} 

\usepackage{balance} 
\usepackage{algorithm}
\usepackage[noend]{algpseudocode}
\usepackage{soul}

\usepackage{multirow}
\usepackage{amssymb}
\usepackage{amsthm}

\newtheorem{thm}{Theorem}
\usepackage{mathtools}
\newtheorem{definition}[thm]{Definition}
\usepackage{enumitem}
\usepackage{xcolor}
\usepackage{newfloat}
\usepackage{subcaption}
\usepackage{listings}
\usepackage{tikz}
\newtheorem{assumption}{Assumption}
\usetikzlibrary{decorations.pathreplacing,calc}

\makeatletter
\gdef\@copyrightpermission{
  \begin{minipage}{0.2\columnwidth}
   \href{https://creativecommons.org/licenses/by/4.0/}{\includegraphics[width=0.90\textwidth]{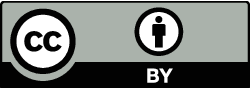}}
  \end{minipage}\hfill
  \begin{minipage}{0.8\columnwidth}
   \href{https://creativecommons.org/licenses/by/4.0/}{This work is licensed under a Creative Commons Attribution International 4.0 License.}
  \end{minipage}
  \vspace{5pt}
}
\makeatother

\setcopyright{ifaamas}
\acmConference[AAMAS '25]{Proc.\@ of the 24th International Conference
on Autonomous Agents and Multiagent Systems (AAMAS 2025)}{May 19 -- 23, 2025}
{Detroit, Michigan, USA}{Y.~Vorobeychik, S.~Das, A.~Nowé  (eds.)}
\copyrightyear{2025}
\acmYear{2025}
\acmDOI{}
\acmPrice{}
\acmISBN{}

\acmSubmissionID{1113}

\title[AAMAS-2025 Formatting Instructions]{Multi-Objective Planning with Contextual Lexicographic \\Reward Preferences}

\author{Pulkit Rustagi}
\affiliation{
   \institution{Oregon State University}
  \city{Corvallis, OR}
  \country{USA}}
\email{rustagip@oregonstate.edu}

\author{Yashwanthi Anand}
\affiliation{
  \institution{Oregon State University}
  \city{Corvallis, OR}
  \country{USA}}
\email{anandy@oregonstate.edu}

\author{Sandhya Saisubramanian}
\affiliation{
   \institution{Oregon State University}
  \city{Corvallis, OR}
  \country{USA}}
\email{sandhya.sai@oregonstate.edu}

\begin{abstract}
Autonomous agents are often required to plan under multiple objectives whose preference ordering varies based on \emph{context}. The agent may encounter multiple contexts during its course of operation, each imposing a distinct lexicographic ordering over the objectives, with potentially different reward functions associated with each context. Existing approaches to multi-objective planning typically consider a single preference ordering over the objectives, across the state space, and do not support planning under \emph{multiple objective orderings within an environment}. We present Contextual Lexicographic Markov Decision Process (CLMDP), a framework that enables planning under varying lexicographic objective orderings, depending on the context. In a CLMDP, both the \emph{objective ordering} at a state and the associated \emph{reward} functions are determined by the context. We employ a Bayesian approach to infer a state-context mapping from expert trajectories. Our algorithm to solve a CLMDP first computes a policy for each objective ordering and then combines them into a single context-aware policy that is valid and cycle-free. The effectiveness of the proposed approach is evaluated in simulation and using a mobile robot. 
\end{abstract}

         
\newcommand{\BibTeX}{\rm B\kern-.05em{\sc i\kern-.025em b}\kern-.08em\TeX}

\begin{document}

\pagestyle{fancy}
\fancyhead{}

\maketitle

\section{Introduction}
\label{sec:introduction}
In many real-world applications, such as navigation~\cite{smith2003fuzzy,fujimura1996path}, and warehouse management~\cite{reehuis2010mixed,gao2022two}, autonomous agents must optimize multiple, potentially competing objectives. The priority over these objectives can be conveniently expressed via a lexicographic ordering, and may vary based on the \emph{context}. Consider a semi-autonomous car navigating in a city (Figure~\ref{fig:context_illustration}), where the possible set of objectives includes minimizing travel time, ensuring pedestrian safety, and minimizing going over uneven road surfaces. When the car navigates through a construction zone, it must prioritize minimizing going over uneven surfaces, followed by pedestrian safety, and finally the speed. When the car is on a highway, it is desirable to prioritize travel time over other objectives, since a highway is designed for high-speed navigation and is typically free of pedestrian traffic and other obstacles. When in urban areas with high foot traffic, ensuring pedestrian safety becomes the highest priority.
Thus, during the course of its navigation, the car must optimize different \emph{objective orderings}, corresponding to the context associated with the region. The reward functions associated with the objectives may also vary based on context. 
\begin{figure}[t]
    \centering
   \includegraphics[width=\linewidth,trim={0.3cm 0.4cm 0.7cm 0},clip]{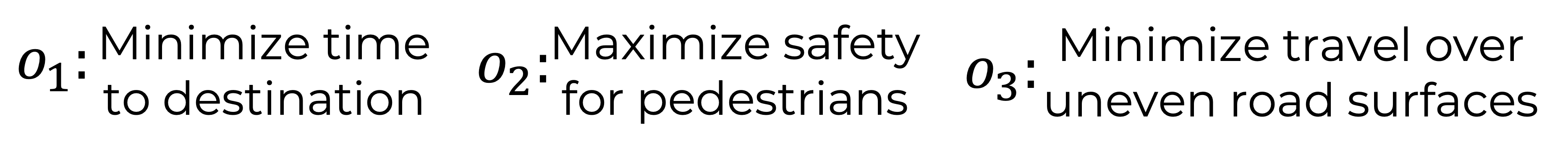}
   \includegraphics[scale=0.38]{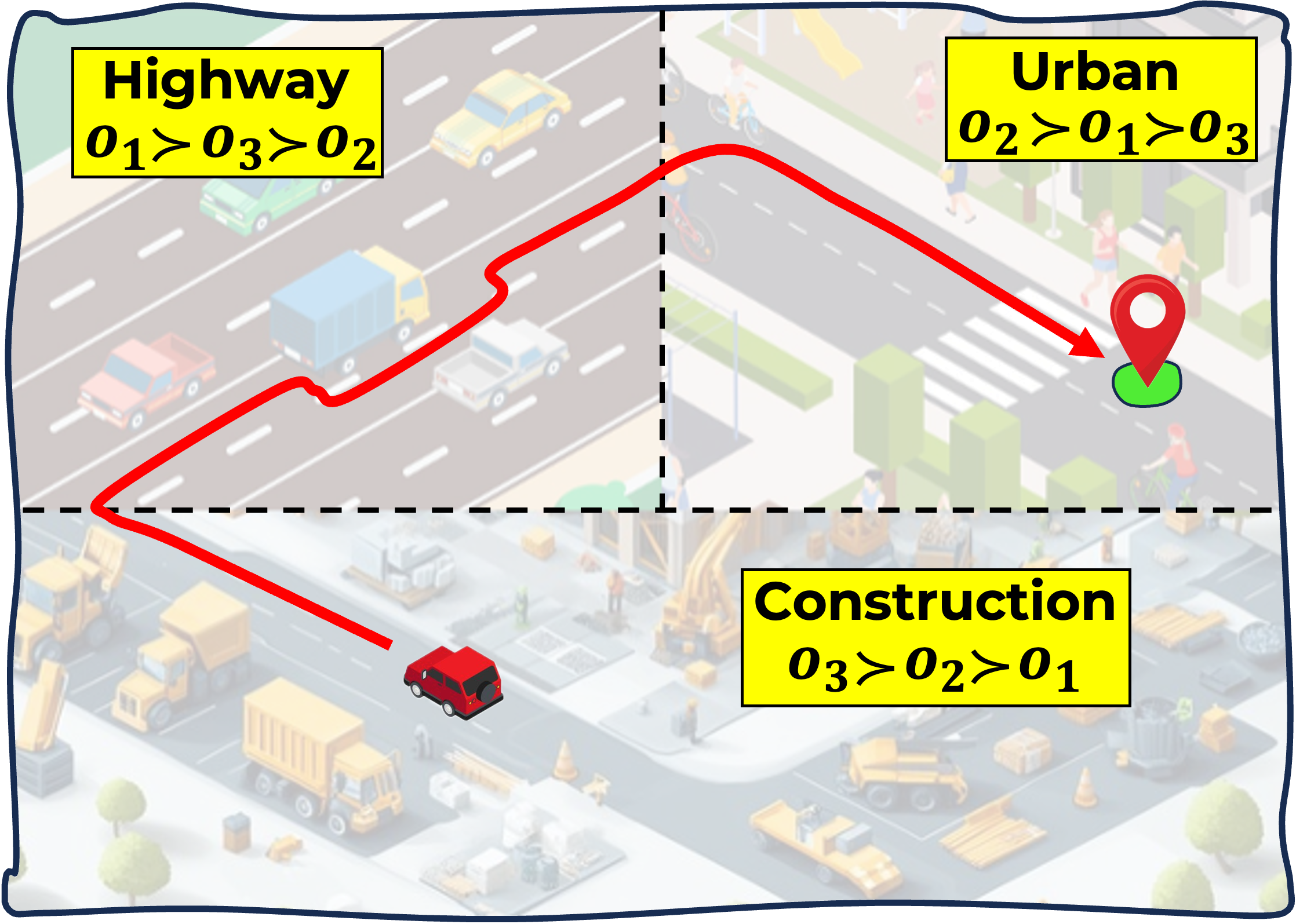}
    \caption{Example of a car navigating in a city environment with three different contexts (\emph{urban}, \emph{highway}, and \emph{construction}), each imposing a unique 
    ordering over the objectives.}
\label{fig:context_illustration}
\end{figure}
Using contextual information for decision-making has often proven to be valuable, such as in contextual bandits~\cite{li2010contextual}, context-aware planning~\cite{kim2023context}, contextual information retrieval and text generation~\cite{hambarde2023information,wang2024utilizing}. The existing literature often defines context as a parameter that influences the environment dynamics and rewards~\cite{benjamins2022contextualize,kim2023context,modi2018markov}, the position of obstacles and other agents~\cite{hvvezda2018context,ter2010context}, or the area of operation~\cite{bahrani2008collaborative}. We define context as a set of \emph{exogenous features} in the state representation that determine the objective ordering and corresponding reward functions for each state.

We consider a lexicographic ordering over objectives, which is a convenient approach for representing relative preferences over multiple objectives. 
Traditional approaches for multi-objective decision-making typically support only a single preference ordering 
over objectives throughout the entire state space. 
A prior work~\cite{wray2015multi} that supports planning with multiple state partitions could potentially be adapted to handle different priority orderings. However, it lacks a systematic method for defining partitions and does not provide an explicit mapping from states to their respective objective orderings. Their framework also does not support varying reward functions for each partition or ensure that the resulting policy is valid and cycle-free.
A recent work~\cite{yang2019generalized} uses a context map to
track the scalarization weights for each objective in every state. However, it does not scale
to larger state spaces.

Our paper addresses these challenges by introducing \emph{Contextual Lexicographic Markov Decision Process} (CLMDP), a framework that supports seamless switching between different objective orderings, determined by the state's context, during the course of agent operation.  We present an algorithm to solve CLMDP by first computing a policy for each objective ordering across the state space, and then combining them into a single context-aware policy, given a state-context mapping. Our algorithm also ensures that the action recommendations across contexts do not conflict with each other, during the merge step. When state-context mapping is unknown, it is derived using Bayesian inference and limited expert trajectories.

\begin{figure*}[t]
    \centering
    \includegraphics[width=0.96\linewidth]{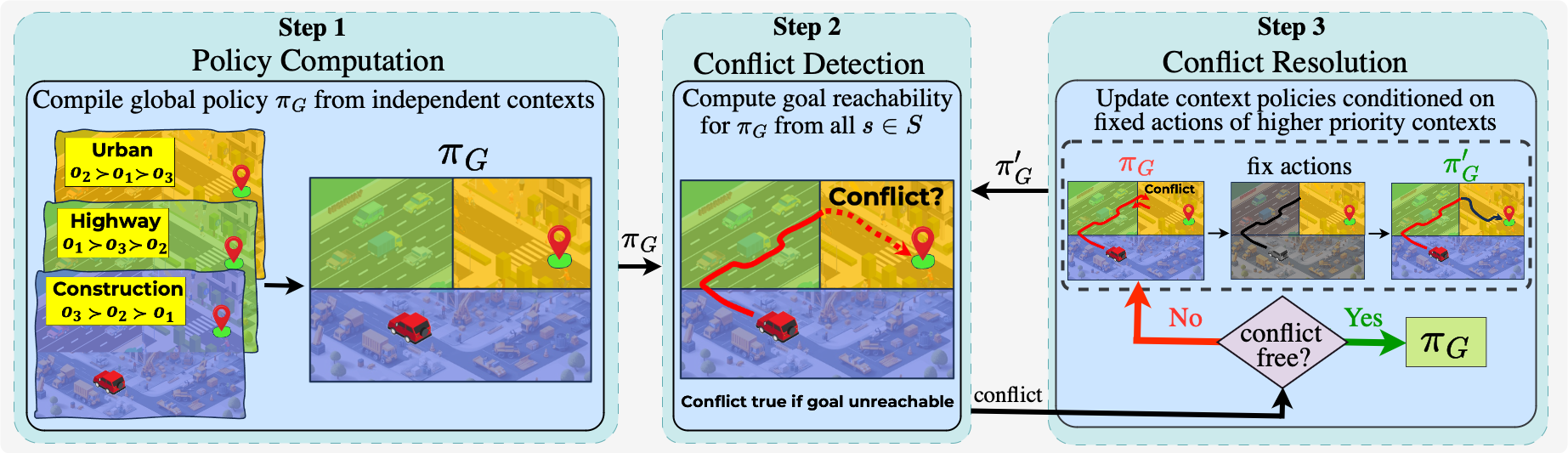}
    \caption{Overview of our solution approach for contextual planning. First, policies are calculated for each context in isolation, across the entire state space, and then compiled into a global policy $\pi_G$ by mapping actions to states based on each state's associated context. Second, $\pi_G$ is analyzed for cycles by estimating goal reachability from each state. Finally, the detected conflicts are resolved by updating lower priority context policies conditioned on fixed actions of higher priority contexts. }
    \label{fig:contextual_approach_overview}
\end{figure*}

Our primary contributions are: (1) formalizing the problem of contextual multi-objective planning as a contextual lexicographic MDP; (2) presenting an algorithm to solve a CLMDP, producing a policy that is free of conflicts in a setting where multiple contexts coexist; (3) theoretical analysis of the algorithm's performance; (4) a Bayesian approach to infer the state-context mapping; and (5) empirical evaluation in simulation and using a physical robot.

\section{Related Work}
\paragraph{Multi-objective planning} Multi-objective planning is gaining increasing attention. Common approaches to solve multi-objective planning problems include scalarization~\cite{wilde2024scalarizing,agarwal2022multi,van2013scalarized}, lexicographic planning~\cite{wray2015multi,saisubramanian2021multi,mouaddib2004multi,skalse2022lexicographic}, constraint optimization~\cite{deb2016multi,deb2001nonlinear}, and Pareto optimization~\cite{aydeniz2024entropy,nickelson2024redefining}. \emph{Scalarization} combines multiple objectives into a single objective value using weighted sum. While it is a popular approach, it often requires non-trivial parameter tuning~\cite{roijers2013survey}. Our work is based on the \emph{lexicographic} formulation that considers a lexicographic ordering over objectives and solves them sequentially.
While the lexicographic Markov decision process (LMDP)~\cite{wray2015multi} can support multiple state partitions, each corresponding to a different objective ordering, it suffers from the following limitations: (1) lack of a principled approach to define state partitions, especially when a state could be associated with multiple contexts such as weather, road conditions, and time of day; (2) hard-coded partitions and lack of explicit state-to-context mapping which makes it difficult to adapt to newer settings where the set of contexts may change; (3) does not support scenarios where the context influences the reward function; and (4) lack of tools to resolve conflicts in action recommendation that arise from solving partitions sequentially, with no principled approach to determine the order. 
The CLMDP framework addresses these shortcomings and facilitates smooth transitions between contexts, avoiding the risk of conflicting actions. The explicit state-to-context mapping in CLMDP offers flexibility and scalability in handling new scenarios, \emph{without} requiring manual interventions to redefine partitions or objective priorities. 

\paragraph{Constraint optimization} approach optimizes a single objective and converts other objectives into constraints~\cite{deb2016multi,deb2001nonlinear}. This approach cannot efficiently balance the trade-offs between different objectives~\cite{deb2010efficient}.
\emph{Pareto} optimization finds non-dominated solutions across objectives~\cite{aydeniz2024entropy,nickelson2024redefining}, making it unsuitable for scenarios where specific preference orderings must be satisfied.

\paragraph{Context-aware planning} 
Prior works use contextual information to determine environment dynamics and rewards~\cite{benjamins2022contextualize,kim2023context,modi2018markov}, or represent specific configurations like obstacle layouts and operational areas~\cite{bahrani2008collaborative,hvvezda2018context,ter2010context}. While many definitions of context exist for different problem settings, we focus our discussion on those most pertinent to multi-objective scenarios. In multi-objective settings, context has been used to assign scalarization weights~\cite{yang2019generalized} but this approach struggles to scale in larger state spaces. We integrate contextual information in a lexicographic setting to enable planning with different objective ordering in different regions of the state space, with associated reward functions determined by the context. 

\paragraph{Bayesian Inference} It is commonly used to infer values of unknown parameters by incorporating prior knowledge and updating beliefs as new information becomes available~\cite{belakaria2020uncertainty,zhi2020online,abdolshah2019multi}.
Bayesian inference is often used in 
inverse reinforcement learning (IRL) to estimate reward functions~\cite{ramachandran2007bayesian,pmlr-v119-brown20a}, and goal inference in multi-agent settings~\cite{ullman2009help,zhi2020online}. In multi-objective optimization problems, Bayesian methods are used to approximate the set of Pareto optimal solutions for competing objectives~\cite{belakaria2020uncertainty,pmlr-v162-daulton22a}, and identify solutions that satisfy user-defined preference ordering over competing objectives~\cite{abdolshah2019multi}.
In this work, we apply Bayesian inference to determine the most likely context of a state, using limited number of expert trajectories.

\section{Contextual Lexicographic MDP}
Consider an agent operating in an environment with multiple  objectives modeled as a  multi-objective Markov decision process (MOMDP). The agent must optimize a lexicographic ordering over the $n$ primitive objectives $\mathbf{o}=\{o_1,\dots,o_n\}$. We focus on goal-oriented MDPs where one of the objectives $o_i$ is to maximize the reward associated with reaching the goal. The lexicographic ordering over objectives at a \emph{particular state} is determined by the \emph{context} associated with it. The context, inferred from a set of \textit{exogenous} features in the state representation, determines the objective ordering at a state and the reward functions associated with the objectives. 
In the example illustrated in Figure~\ref{fig:context_illustration}, the set of contexts in the environment are: highway, urban area, and construction zone.
Each context imposes a unique ordering over the objectives, which determines an acceptable behavior for that context. Since an agent may encounter multiple contexts during the course of its operation, its reasoning module must be able to seamlessly switch between optimizing different objective orderings.

We introduce \emph{Contextual Lexicographic Markov Decision Process (CLMDP)}, a framework that facilitates optimizing different objective orderings in different regions of the state space. 

\begin{definition} A contextual lexicographic Markov decision process (CLMDP) is denoted by $M = \langle \mathcal{C},\Omega,\mathbf{o},\mathbf{w},f_{\mathbf{w}}, S, A, T, \mathbf{R},f_R, \mathcal{Z}\rangle$, where:
    \begin{itemize}[noitemsep]
        \item $\mathcal{C}=\{c_1,\dots,c_m\}$ denotes a finite set of contexts;
        \item $\Omega$ is a lexicographic ordering over contexts in $\mathcal{C}$;
        \item $\mathbf{o}=\{o_1,\dots,o_n\}$ denotes the set of primitive objectives;
        \item $\mathbf{w}=\{\omega_1,\dots,\omega_m\}$ is a set of unique lexicographic orderings over all primitive objectives, with $\omega_i$ denoting a  lexicographic ordering over all objectives such as $o_1\succ \ldots \succ o_n$, $o_i\in \mathbf{o}$;
        \item $f_{\mathbf{w}}: \mathcal{C}\to\mathbf{w}$ is a function that maps a context $c_j\in\mathcal{C}$ to a lexicographic ordering over the primitive objectives $\omega_i\in\mathbf{w}$;
        \item $S$ is a finite set of states, with initial state $s_0\in S$;
        \item $A$ is a finite set of actions;
        \item $T:S\times A \times S \rightarrow [0,1]$ is the transition function that determines the probability of reaching state $s'$ from state $s$ by executing action $a$, independent of the context; 
        \item $\mathbf{R}\!=\![R_{1},\dots,R_{n}\!]\!\in\!\mathbb{R}^n$ is the space of possible reward vectors, where for each $ o_i\in \mathbf{o}$,  $R_{i}\!:\!S\!\times\!A\!\rightarrow\!\mathbb{R}\,$; 
        \item $f_R:\mathcal{C}\!\to\!\mathbf{R}$ maps context $c_j\in \mathcal{C}$ to a vector of context-specific reward functions, $f_R(c_j)\!=\![R_{1,j},\dots, R_{n,j}]\in \mathbf{R}$; and
        \item $\mathcal{Z}: S \to \mathcal{C}$ is a deterministic mapping of each $s\in S$ to  $c\in\mathcal{C}$
    \end{itemize}
\end{definition}

In a CLMDP, a state can belong to multiple contexts simultaneously, when some state features are shared across contexts. For example in Fig.~\ref{fig:context_illustration}, a road segment with uneven road surface may be associated with both urban  and construction zone contexts.
In such cases, $\mathcal{Z}$ maps the state to the context with higher priority in the meta-ordering $\Omega$. This ensures that the agent's decision-making process remains consistent and aligned with the most critical objective at that state. When $\mathcal{Z}$ maps every state to the same context $c\in \mathcal{C}$, the CLMDP becomes an MDP with a single lexicographic ordering over the objectives as in~\cite{wray2015multi}. In this paper, we consider $\mathcal{Z}$ to be stationary and deterministic but the framework can be extended to support non-stationary $\mathcal{Z}$.

The following section presents an algorithm to solve a CLMDP, given $\mathcal{Z}$. In Section~\ref{sec:learning_Z}, we present a Bayesian approach to infer $\mathcal{Z}$.

\section{Solution Approach}
We begin with an overview of our solution approach, illustrated in Figure~\ref{fig:contextual_approach_overview}.
Given $\mathcal{Z}$, our approach solves CLMDPs by computing a policy $\pi_i$ for each context $c_i \in \mathcal{C}$ independently, and then combines them into a global policy $\pi_G$, based on the context associated with each state. Combining different $\pi_i$ into $\pi_G$ may result in cycles as each policy is computed independent of other policies and contexts. The cycles are detected and resolved by updating the policies associated with lower priority contexts, based on $\Omega$.

\subsection{Computing global policy $\pi_G$} A policy $\pi_i$ for each context $c_i \in \mathcal{C}$ is first computed independent of other contexts. We do this by considering each $c_i \in \mathcal{C}$ to be the \emph{only} context across the entire state space and use the corresponding lexicographic ordering over objectives ($f_{\mathbf{w}}(c_i)$) and the associated reward functions ($f_R(c_i)$) over \emph{all states}. This multi-objective problem, with a single ordering over the objectives, is solved using lexicographic value iteration algorithm~\cite{wray2015multi}. Each $\pi_i$ specifies the optimal actions for its respective context. These individual policies are then compiled into a \emph{global policy} $\pi_G$, where actions are mapped to each state, based on the actual context of each state, following state-to-context mapping $\mathcal{Z}$. 

\subsection{Conflict Detection}\label{sec:detect}
Combining multiple policies, computed under different contexts and associated objective orderings, can lead to conflicting action recommendations that result in cycles and affect goal reachability.  

\begin{definition}\label{def:conflict}
A policy $\pi_G$, in a goal-oriented MDP with goal state $s_g$, is said to have a \textbf{conflict} if there exists at least one state $s$ from which the probability of reaching the goal, $ \emph{Pr}(s_g|s,\pi_G)=0, \exists s \in S.$
\end{definition}

\begin{definition}\label{def:conflict-free}
A \textbf{conflict-free policy} $\pi_G$ has a non-zero probability of reaching the goal from every state $s$, $\emph{Pr}(s_g|s,\pi_G)>0, \forall s\in S$.
\end{definition}

\begin{assumption}\label{assumption:policy-exist}
    There exists a conflict-free policy under $\Omega$.
\end{assumption}
We assume that there exists an underlying conflict-free policy for CLMDP and any cycles detected are due to $\pi_G$ computation based on policies calculated for individual contexts.
To detect conflicts that introduce dead-ends, Algorithm~\ref{alg:policy_conflict_checker} estimates the goal reachability for each state, under $\pi_G$.  
The goal reachability is calculated using a Bellman backup-based method (lines \texttt{11-18}) with a reward function $R_e$ defined as follows: \[R_e(s) = \begin{cases}
    +1,& \text{for }s=s_{goal}\\
     0,&otherwise. 
\end{cases}\] 

States with $V_r(s) = 0$ do not have a path to the goal, indicating a conflict (line \texttt{25}). Thus, if $V_r(s) = 0$ for any $s\in S$, it indicates a conflict.  Conversely, if all states have $V_r > 0$, the policy is conflict-free.
To handle large state spaces, a modified value iteration is used, marking states as solved once they individually converge (lines \texttt{19-22}). Since a converged state's value no longer changes, skipping further updates ensures computational efficiency without compromising the correctness of the solution, allowing for early termination.
For problems with very long trajectories, Bellman backups can be executed in logarithmic space to avoid false positives caused by product over low values~\cite{wray2016log}. Algorithm~\ref{alg:policy_conflict_checker} outputs a Boolean variable indicating the presence of conflicts, along with all the states from which the goal is unreachable.

\begin{algorithm}[t]
\caption{ConflictChecker}
\label{alg:policy_conflict_checker}
\begin{algorithmic}[1]
\State \textbf{Input} Global policy $\pi_G$, state space $S$, transition function $T$, discount factor $\gamma$, goal state $s_g$
\State \textbf{Initialize} conflict $\gets$ \texttt{False}, $\delta\gets 0$, $\epsilon\gets 10^{-6}$
\State \textbf{Initialize} $V_r \gets \{s:0\,|\,s\in S\}$, $\Delta\gets \{s:0\,|\,s\in S\}$
\State \textbf{Initialize} $R_e \gets \{s:0\,|\,s\in S\}$, $R_e(s_g) \gets 1.0$
\State \textbf{Initialize} $S_{solved}\gets [\,]$, $S_{unsolved}\gets S$, $S_{\emph{conflict}}\gets [\,]$
\State $S_{solved}.\textsc{Append}(s_g)$ \Comment{{\color{gray}only $s_g$ is solved initially}}
\State $S_{unsolved}.\textsc{Remove}(s_g)$
\While{$S_{unsolved}\not=\phi$}\Comment{{\color{gray}continue if unsolved states remain}} 
    \State $V_r'\gets V_r$
    \State $S_{check}\gets [\,]$
    \For{$s\in S_{unsolved}$}\Comment{{\color{gray}update $V_r$ for all unsolved states}} 
    \State $V_r(s)\gets R_e(s)$
        \For{$s'\in S\textbf{ }\And\textbf{ }T(s,\pi_G(s),s')>0$}
            \State $V_r(s)\gets V_r(s)+\gamma \cdot T(s,\pi_G(s),s')\cdot V_r(s')$
            \If{$s'\in S_{solved}\textbf{ }\And\textbf{ }s\not\in S_{check}$}
                \State $S_{check}.\textsc{Append}(s)$\Comment{{\color{gray}if $s$ leads to a solved state}}
            \EndIf
        \EndFor
        \State $\delta\gets \max(\delta,|V_r'(s)-V_r(s)|)$
        \State $\Delta(s)\gets |V_r'(s)-V_r(s)|$
    \EndFor
    \For{$s\in S_{check}$}
        \If{$\Delta(s)<\epsilon$}\Comment{{\color{gray}label $s$ as solved if $V_r(s)$ converged}} 
                \State $S_{solved}.\textsc{Append}(s)$
                \State $S_{unsolved}.\textsc{Remove}(s)$
        \EndIf
    \EndFor
    \If{$\delta<\epsilon$ \textbf{or} $S_{unsolved}=\phi$}
         \State break
    \EndIf
\EndWhile
\State $S_{\emph{conflict}} \gets [s\,|\,s\in S \text{ if }V_r(s)=0]$\Comment{{\color{gray}list of conflict states}}
\If{$S_{\emph{conflict}}\not=\phi$}
    \State conflict$\,\gets\texttt{True}$
\EndIf
\State \Return conflict, $S_{\emph{conflict}}$
\end{algorithmic}
\end{algorithm}

\begin{algorithm}[t]
\caption{ConflictResolver}
\label{alg:policy_conflict_resolver}
\begin{algorithmic}[1]
\State\textbf{Input} policy $\pi_G$, set of contexts $\mathcal{C}$, state to context mapping function $\mathcal{Z}$, context to objective ordering map $f_{\mathbf{w}}$, meta-ordering over contexts $\Omega\equiv c_1\succ\dots\succ c_m$, $S$, $A$, $T$, \textbf{R}, $\gamma$, goal $s_g$, list conflict states from ConflictChecker $S_{\emph{conflict}}$
\State \textbf{Initialize} $A_{new}\gets \{\}$, $\Pi\gets \{\}$, $\bf{c^*}\gets c_1$ 
\If{conflict = \texttt{False}} \Return $\pi_G$\EndIf
\For{$s\in S_{\emph{conflict}}$}
\If{$c^*\succ_{\small{\Omega}} \mathcal{Z}(s)$} \State $c^*\gets \mathcal{Z}(s)$\Comment{{\color{gray}stores lowest priority conflict context}}
\EndIf
\EndFor
\For{$c_i$ varying from $c^*$ to $c_1$} \Comment{{\color{gray}low to high priority}} 
\State $C_{update}\gets \{\}, A_{new}\gets A$
\For{$c_j$ varying from $c_i$ to $c_m$} 
    \State $C_{update} \gets C_{update}\cup c_j$ \Comment{{\color{gray}contexts to be updated}} 
\EndFor
\For{$s\in S$ \textbf{and} $\mathcal{Z}(s)\in \mathcal{C}\setminus C_{update}$}
    \State $A_{new}[s]\gets \pi_G[s]$\Comment{{\color{gray}fix actions for $s\not\in C_{update}$}} 
\EndFor
\For{$c_j$ varying highest to lowest priority in  $C_{update}$}
    \State $\Pi[c_j]\gets \emph{LVI}(S,A_{new},T,\textbf{R},f_{\mathbf{w}}(c_j))$\Comment{{\color{gray}new policy for $c_j$}} 
    \For{$s\in S \textbf{ and }\mathcal{Z}(s)=c_j$}
        \State  $A_{new}[s]\gets \Pi[c_j][s]$\Comment{{\color{gray}fix actions for $s\in c_j$}} 
        \State $\pi_G[s]\gets \Pi[c_j][s]$
    \EndFor
\EndFor
\State conflict, \textunderscore $\gets$ ConflictChecker($\pi_G$, $S$, $T$, $\gamma$, $s_g$)
\If{no conflict} 
    \Return{$\pi_G$} 
\EndIf
\EndFor
\State \Return{$\pi_G$} 
\end{algorithmic}
\end{algorithm}

\subsection{Conflict Resolver}
To resolve conflicts identified by Algorithm~\ref{alg:policy_conflict_checker}, conflict resolver in Algorithm~\ref{alg:policy_conflict_resolver} updates policies starting from the lowest-priority context involved in the conflict and moves upwards to the highest priority, while keeping the policies of higher-priority contexts fixed. The algorithm begins by identifying a set of contexts requiring updates, denoted as $C_{\emph{update}}$ (lines \texttt{4-10}), starting with the lowest priority context involved in the conflict, based on $\Omega$. This set expands \emph{iteratively}, adding higher-priority contexts to $C_{\emph{update}}$ if conflicts persist.
At each step, actions for states in higher-priority contexts outside of $C_{\emph{update}}$ are fixed according to $\pi_G$, creating an updated action space $A_{\text{new}}$ (lines \texttt{11-12}). This ensures that these contexts remain unaffected during updates to other contexts. The policies for contexts in $C_{\emph{update}}$ are then updated sequentially from highest to lowest priority, following $\Omega$ and using $A_{\text{new}}$ to condition policy updates on the fixed actions (line \texttt{14}). After updating the policy for a context in $C_{\emph{update}}$, the actions for states associated with that context are fixed, and the action space is updated (lines \texttt{16-17}).

Once all contexts in $C_{\emph{update}}$ are updated, Algorithm~\ref{alg:policy_conflict_checker} is invoked to check for conflicts (line \texttt{18}). If the conflict is resolved, the new conflict-free policy is returned (line \texttt{19}). 
Otherwise, the process is  repeated with an expanded $C_{\emph{update}}$ that includes the next higher-priority context, using the loop variable $c_i$ (line \texttt{7}). 
In the worst case, during the last iteration (when $c_i = c_1$), all contexts are updated starting with highest-priority context, sequentially followed by all lower-priority contexts, ensuring alignment with $\Omega$. Algorithm~\ref{alg:policy_conflict_resolver} is guaranteed to find a conflict-free policy, when   one exists.

\subsection{Theoretical Analysis}
In this section, we analyze the correctness and time complexities of Algorithms~\ref{alg:policy_conflict_checker} and~\ref{alg:policy_conflict_resolver}. 

\begin{proposition}
    Algorithm~\ref{alg:policy_conflict_checker} correctly identifies conflicts.
\end{proposition}  
\begin{proof} We first describe $V_r$ convergence and then show that 
states from which the goal cannot be reached have $V_r(s) = 0$ under reward $R_e$ defined in Sec.~\ref{sec:detect}. 
Since Algorithm~\ref{alg:policy_conflict_checker} uses Bellman backup (lines \texttt{12-14}), $V_r$ is guaranteed to converge~\cite{sutton1998introduction}.

Let $S_u$ denote the set of all states from which $s_{goal}$ is unreachable, and let $S_r$ represent the set of all states that can reach the goal, such that $S = S_u \cup S_r$. In the worst case, the unsolved set may be the entire state space, $S_{unsolved}=S$. $V_r$ is first initialized to zero for all states (line \texttt{3}): $V_r^{(k=0)}(s) = 0,  \forall s \in S$. We now show using induction that $V_r(s_u)=0, \forall s_u \in S_u$ always. For $k=0, V_r^{(0)}(s_u)=0$ (by initialization). Using Bellman equation, the value of a state $s$ in the $k^{th}$ iteration is updated as:
\begin{equation}
    V_{r}^{(k)}(s) = R_e(s) + \gamma \max_a \sum_{s' \in S} T(s,a,s') V_r^{(k-1)}(s')
    \label{eqn:re_bellman}
\end{equation}
where $T(s,a,s')$ is the transition probability from $s$ to $s'$ using action $a$, and $\gamma$ is the discount factor. Note that, while all actions in a state $s$ are considered here for generality, Algorithm~\ref{alg:policy_conflict_checker} only requires $\pi_G(s)$ for conflict detection. By Definition~\ref{def:conflict}, 
$T(s_u, a, s_r)\!=\!0$ and so the values of any $s_u\in S_u$ never get updated. Hence for $k^{th}$ iteration: $V_r^{(k)}(s_u)=0$. For iteration $k+1$, using Equation~\ref{eqn:re_bellman},
\[ V_{r}^{(k+1)}(s_u) = R_e(s_u) + \gamma \max_a \sum_{s' \in S} T(s_u,a,s') V_r^{(k)}(s'), \forall s_u \in S_u.\]

For $V_{r}^{(k+1)}(s_u)$ to be non-zero, at least one of its successors must have a path to goal, $s' \in S_r$. 
Substituting $R_e(s_u)=0$ and $T(s_u,a,s_r)\!=\!0$ in the above equation, we get $V_{r}^{(k+1)}(s_u)\!=\!0$.
Thus, Algorithm~\ref{alg:policy_conflict_checker} identifies a conflict, when there exists one.
\end{proof}

\begin{proposition}
The time complexity of Algorithm~\ref{alg:policy_conflict_checker} is $O(|S|^2)$. \label{prop:alg1_time}
\end{proposition}
\begin{proof}
The complexity of Algorithm~\ref{alg:policy_conflict_checker} depends on the time taken for calculating $V_r$ and the time taken to check for conflicts based on the computed $V_r$.  Computing $V_r$ (lines \texttt{8-24}) is essentially policy evaluation with $\pi_G$ under reward $R_e$. Hence the complexity of this step is $O\left(|S|^2\right)$. The worst case complexity of conflict detection step (lines\texttt{ 25-27}) is $O\left(|S|\right)$  since it cycles through the entire state space. Thus, the worst case time complexity of Algorithm~\ref{alg:policy_conflict_checker} is $O\left(|S|^2+|S|\right)\implies O\left(|S|^2\right)$.
\end{proof}

\begin{proposition}
    Algorithm~\ref{alg:policy_conflict_resolver} guarantees a conflict-free policy.
\end{proposition}

\begin{proof}
We prove this using the property of lexicographic ordering, where the optimality of a policy in a higher-priority context does not depend on those of lower priority. Consider a meta-ordering of contexts $\Omega \equiv c_1 \succ \dots \succ c_m$. Algorithm~\ref{alg:policy_conflict_resolver} identifies the lowest-priority context involved in the conflict, $c^*$ (lines \texttt{4-6}). The algorithm then marks all contexts from the identified $c^*$ (in line \texttt{6}) to least priority context $c_m$ for updates (lines \texttt{7-10}). Higher-priority contexts are kept fixed (lines \texttt{11-12}) while the policies for the identified contexts are updated using lexicographic value iteration (line \texttt{14}). After the context's policy is updated, the actions are fixed (lines \texttt{15-17}) before moving on to lower priority contexts following $\Omega$. 

We now show that the final iteration, where all context policies are updated in sequence according to $\Omega$, yields a conflict-free policy. Specifically, considering policy update in line \texttt{14}:
\[
\pi_{c_i}(s) = \arg\max_{a \in A^*_{c_i}} \left[ R(s) + \gamma \sum_{s' \in S} T(s, a, s') V_{c_i}(s') \right], \quad \forall 2 \leq i \leq m 
\]
Here, $A^*_{c_i}$ represents the action space, with actions fixed for states in the higher-priority contexts $c_1$ through $c_{i-1}$. Since these updates do not affect the higher-priority contexts, no new conflicts can be introduced in $c_1, \dots, c_{i-1}$. Thus, the sequential updates across all contexts, respecting the lexicographic ordering $\Omega$, ensure that no conflicts arise in higher-priority contexts. The resulting policy is also verified for conflicts using Algorithm~\ref{alg:policy_conflict_checker} (line \texttt{18}), ensuring that Algorithm~\ref{alg:policy_conflict_resolver} terminates with a conflict-free policy.
\end{proof}

\begin{proposition}
    The time complexity of Alg.~\ref{alg:policy_conflict_resolver} is $O(|\mathcal{C}|^2 |S|^2 |A|)$.
\end{proposition} 
\begin{proof}
The time complexity of Algorithm~\ref{alg:policy_conflict_checker} can be broken down into three primary components: identifying the minimum priority context involved in a conflict (lines \texttt{4-6}), updating the policy based on the lexicographic ordering over conflicts (lines \texttt{8-17}), and verifying that the updated policy is conflict-free (line \texttt{19}). The time complexity of identifying the minimum priority context (lines \texttt{4-6}) is $O(|S|)$ time in the worst case. The time taken by the main outer loop (lines \texttt{7-19}) is  $O(|\mathcal{C}|)$ since it iterates over all contexts in $\mathcal{C}$. Within this loop, compiling the list of contexts to be updated (lines \texttt{9-10}) requires $O(|\mathcal{C}|)$ time as it cycles through all contexts. Fixing actions for the remaining contexts (lines \texttt{11-12}) takes $O(|S|)$ time in the worst case if actions need to be fixed for all states. Updating policies for contexts in $C_{\emph{update}}$ (lines \texttt{13-17}) involves iterating through all contexts in $C_{\emph{update}}$, which takes $O(|\mathcal{C}|)$ time. For each context, the policy computation using lexicographic value iteration takes $O(|S|^2 |A|)$, and the subsequent policy update and fixing loop requires $O(|S|)$ time. Therefore, this step takes $O\left(|\mathcal{C}|\left(|S|^2 |A| + |S|\right)\right)$. Finally, conflict checking (line \texttt{18}) has a time complexity of $O(|S|^2)$ (from Proposition~\ref{prop:alg1_time}). Thus, the overall time complexity of Algorithm~\ref{alg:policy_conflict_resolver} is $O\left(|\mathcal{C}|^2 |S|^2 |A|\right)$.
\end{proof}

\section{Learning State-Context Mapping}
\label{sec:learning_Z}

The mapping between the states and contexts, $\mathcal{Z}$ is critical for effective planning but this information may sometimes be unavailable to the agent a priori. We address this challenge by using a Bayesian approach to infer the likely context of a state, using a limited number of expert demonstrations. Note that we assume access to all other parameters of CLMDP, except $\mathcal{Z}$. 

Consider a set of expert trajectories $\boldsymbol{\tau}\!=\!\{\tau_1, \ldots, \tau_N\}$, where each $\tau_i$ is a trajectory that originates at a random start state and terminates at the goal state $s_g$,  $\tau_i(s)\!=\!\{(s, a_0),\ldots,(s_g, a_n)\}$. Given $\boldsymbol{\tau}$, the posterior probability of a context $c$ is computed as,
\begin{equation}
Pr(c|s,\vec{r},\boldsymbol{\tau}) \propto \sum \limits_{a \in A} Pr(\vec{r}|s,a,c)\cdot Pr(a|s,c,\boldsymbol{\tau})\cdot Pr(c|\boldsymbol{\tau})
\label{eqn:posterior}
\end{equation}
where $Pr(\vec{r}|s,a,c)$ is the probability of observing the reward vector $\vec{r}$ when executing action $a$ in state $s$ under context $c$, $Pr(a|s,c,\boldsymbol{\tau})$ is the probability of the expert executing action $a$ in state $s$ if $c$ is the underlying context associated with $s$, and $Pr(c|\boldsymbol{\tau})$ denotes the prior probability, given $\boldsymbol{\tau}$. 
We marginalize over actions to ensure that the computed posterior reflects the overall probability of context $c$ across all potential actions that the expert could have taken, given the state and reward (which varies based on the context).
We first describe how these terms are calculated for state-action pairs in the dataset and then discuss how to estimate the posterior for other states. 
Algorithm~\ref{alg:infer_Z} outlines the steps involved in inferring $\mathcal{Z}$.

\paragraph{Likelihood estimation:} The probability of observing a particular reward vector $\vec{r}$ in the data is either one or zero, based on $f_R$ in the CLMDP (line \texttt{10}). That is, $Pr(\vec{r}|s,a,c)\!=\!1$ when $\vec{r}$ matches the reward vector under context $c$ ($f_R(c)$), and $Pr(\vec{r}|s,a,c)\!=\!0$ otherwise. 

\begin{algorithm}[t]
\caption{Infer State-Context Mapping, $\mathcal{Z}$}
\label{alg:infer_Z}
\begin{algorithmic}[1]
    \State \textbf{Input} $\mathcal{C}$, $\mathbf{o}$, $f_{\mathbf{w}}$, $\Omega$, $S$, $A$, $T$, $\mathbf{R}$, $s_g$, Set of expert trajectories $\boldsymbol{\tau}$
    \State $\mathcal{P}(s,a)\gets [], \forall (s,a) \in \boldsymbol{\tau}$
    \For{$c \in \mathcal{C}$}
        \State $\pi_c\gets LVI(S,A,T,\mathbf{R},f_{\mathbf{w}}(c))$ \Comment{{\color{gray} compute policy for $c$}}
        \State $\mathcal{P}(s,\pi_c(s)).\textsc{Append}(c), \forall (s,\pi_c(s))\in\boldsymbol{\tau}$\Comment{{\color{gray} possible contexts}}
    \EndFor
    \For{$s \in S$}
        \State $\mathcal{Z}(s)\gets\{\}$
        \For{$c \in \mathcal{C}$}
            \For{$a \in A$}
                \State Calculate $Pr(\vec{r}|s,a,c)$ based on $f_R(c)$
                \If{$s \in \boldsymbol{\tau}$} \Comment{{\color{gray} state is in expert data}}
                    \State $a_E\gets$ expert action in $s$
                    \State Calculate $Pr(a=a_E|s, c \in \mathcal{P}(s,a))$
                    \State $Pr(c)\gets Pr(c|\mathcal{P}(s,a))$
                \EndIf
                \If{$s \notin \boldsymbol{\tau}$} \Comment{{\color{gray} state is not in expert data}}
                    \State Calculate $Pr(a=\pi_c(s)|s,c)$ 
                    \State $Pr(c)\gets$uniform distribution over $\mathcal{C}$
                \EndIf
            \EndFor
            \State Compute posterior $Pr(c|s,\vec{r})$ using Eqn.~\ref{eqn:posterior}
        \EndFor
        \State $\hat{c}\!\gets\!(c\!\in\!\mathcal{C}: Pr(c|s,r)\!=\!\max\limits_{c'\in\mathcal{C}}(Pr(c'|s,r)))$ 
        \State $\mathcal{Z}(s)\!\gets\!\text{argmax}_{\hat{c}}\Omega$ 
    \EndFor
    \State \Return{Z}
\end{algorithmic}
\end{algorithm}

\paragraph{Action probability:} For a state $s$ in the expert data, the probability of taking an action $a$ under context $c$ is either one or zero depending on whether the expert followed that action during demonstration. To estimate this probability, we calculate a set of \emph{possible contexts}, $\mathcal{P}(s_k,a_k)$, that the expert might have followed for each state-action pair in the data, as follows. 
First, we compute a policy $\pi_c$ for each $c \in \mathcal{C}$, assuming $c$ is the only context and \emph{all} states are mapped to it (line \texttt{4}). 
The objective ordering and the reward functions associated with $c$ are determined using $f_{\mathbf{w}}(c)$ and $f_R(c)$ respectively. Note that this problem has a single context (and hence a single ordering over primitive objectives) and is solved using lexicographic value iteration (LVI)~\cite{wray2015multi}. Second, the expert's action for a state $s$ in the data is compared with $\pi_c(s)$. A context $c$ is considered as a potential context if $\pi_c(s)$ matches expert action at state $s$ (line \texttt{5}). The action probability $Pr(a|s,c,\boldsymbol{\tau})\!=\!1$, if the expert followed action $a$ and  $c \in \mathcal{P}(s,a)$. Otherwise, $Pr(a|s,c,\boldsymbol{\tau})\!=\!0$ (line \texttt{13}). 

\paragraph{Calculating informed prior:}
The prior probability $Pr(c|\boldsymbol{\tau})$ is often a uniform distribution. However, when expert data is available, an \emph{informed prior} can be calculated (line \texttt{14}). The prior $Pr(c|\boldsymbol{\tau})$ for a state-action pair $(s,a)$ in the dataset is calculated as the fraction of occurrences of $c$ in the possible contexts set $\mathcal{P}(s,a)$.

\paragraph{Posterior estimation for states not in dataset:}
For states not in $\boldsymbol{\tau}$, the posterior is estimated using $ Pr(c|s,\vec{r}) \propto \sum \limits_{a \in A} Pr(\vec{r}|s,a,c)\cdot Pr(a|s,c)\cdot Pr(c)$, where each term is calculated based on the policy computed for each context, instead of using expert data.
The action probability is determined by the policy $\pi_c$ that is computed by considering all states mapped to context $c$. $Pr(a|s,c)\!=\!1$ if $\pi_c(s)=a$, and $Pr(a|s,c)\!=\!0$ otherwise (line \texttt{16}). 
The prior $Pr(c)$ is a uniform distribution over all contexts (line \texttt{17}). 
$Pr(\vec{r}|s,a,c)$ is determined based on $f_R$, similar to estimating it under expert data (line \texttt{10}).

Finally, the context with the highest posterior probability is mapped to the state $s$. If multiple contexts are equally probable, then the context with a higher priority in the meta-ordering $\Omega$ is assigned to the state (lines \texttt{19-20}), ensuring that the most critical objective in that state is prioritized. 
A natural question then is why we need expert data if $\mathcal{Z}$ can be inferred without it. 
When $\mathcal{Z}$ is inferred without using expert data,
it does not provide any information as to which context should be followed in the states.  Expert data offers insights into potential contexts for a subset of states, enabling the agent to compute a conflict-reduced policy.

\begin{figure*}[t]
    \centering
    \includegraphics[width=\linewidth,trim={0 1cm 0 1cm},clip]{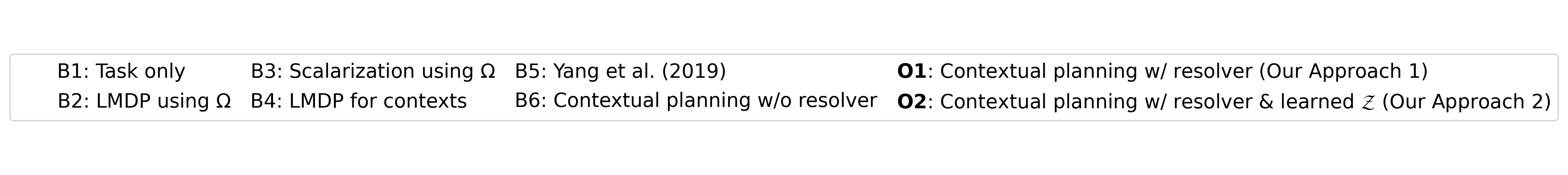}
    \begin{subfigure}[t]{0.33\textwidth}
        \includegraphics[width=\linewidth,trim={0 0 1.88cm 0},clip]{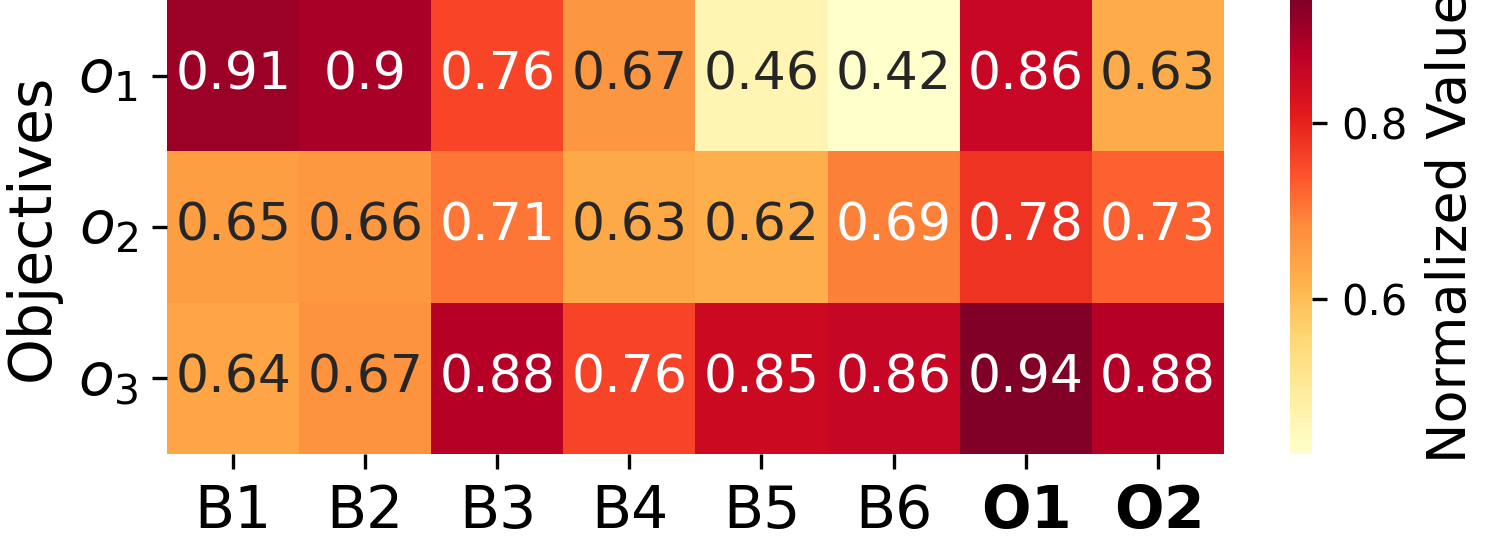}
        \caption{Salp domain}
        \label{fig:salp_heatmap}
    \end{subfigure}
    \begin{subfigure}[t]{0.31\textwidth}
        \includegraphics[width=\linewidth,trim={0.6cm 0 1.88cm 0},clip]{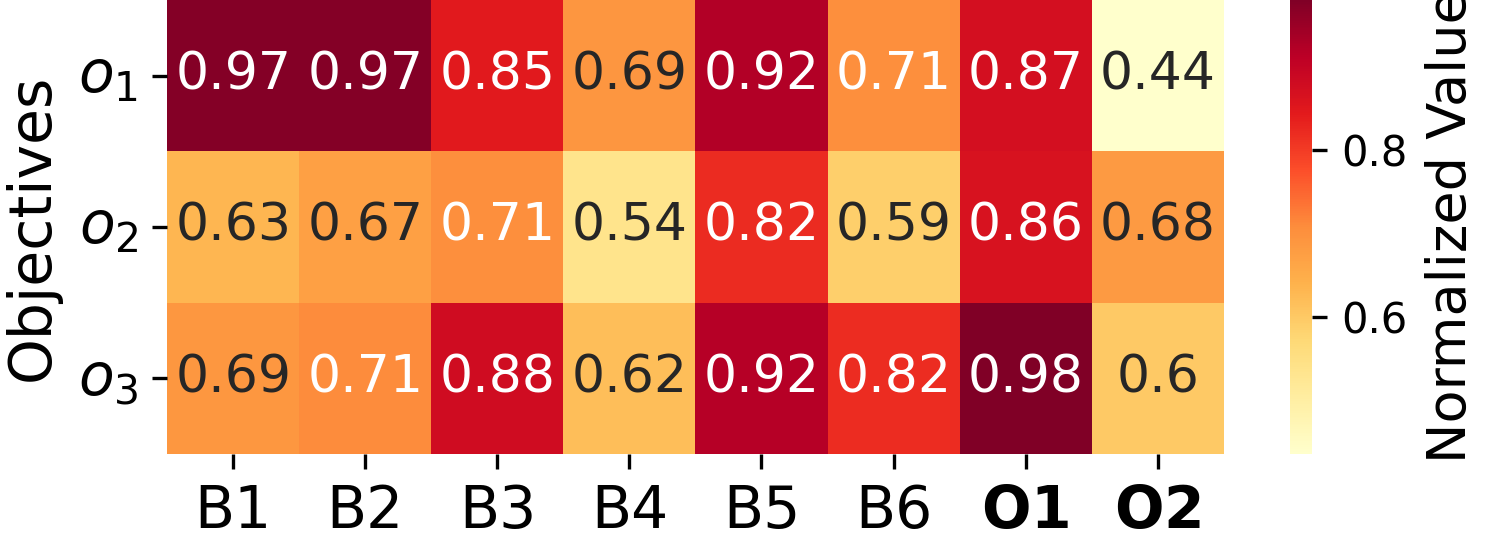}
        \caption{Semi-autonomous taxi domain}
        \label{fig:taxi_heatmap}
    \end{subfigure}
    \begin{subfigure}[t]{0.352\textwidth}
        \includegraphics[width=\linewidth,trim={0.6cm 0 0.2cm 0},clip]{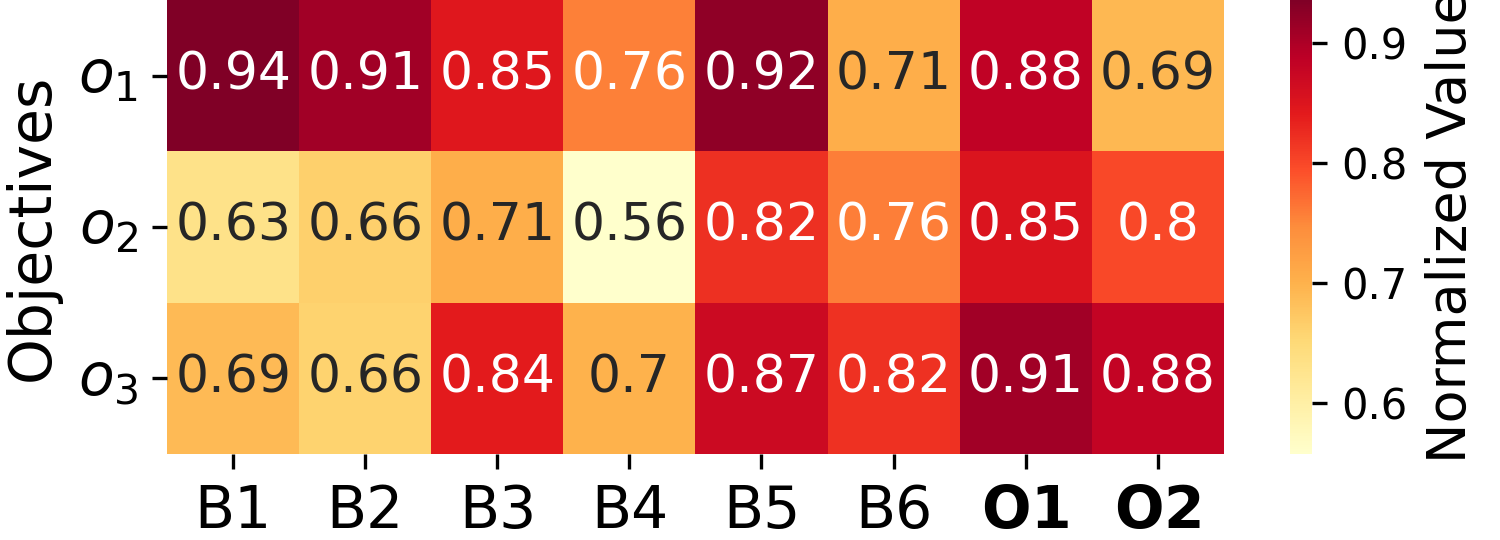}
        \caption{Warehouse domain}
        \label{fig:warehouse_heatmap}
    \end{subfigure}
    \caption{Performance in all objectives ($o_1,\,o_2,\,o_3$) normalized against maximum reward, averaged over 100 trials across 5 instances of each domain, represented as a heatmap with darker shade denoting higher and lighter denoting lower reward. }
    \label{fig:heatmap_all_domains}
\end{figure*}

\section{Experiments}
We evaluate our approach against six baselines using three domains in simulation and validate our results through hardware experiments. Unless specified otherwise, all algorithms were implemented by us in Python and the simulations were run on an iOS machine with $18$ GB RAM. 
\footnote{Code: \url{https://tinyurl.com/Contextual-LMDP}.}

\paragraph{Baselines} 
We evaluate the performance of two variants of our approach, \emph{contextual planning with conflict resolver}: one with a provided state-to-context mapping $\mathcal{Z}$ (\textbf{O1}), and the other where $\mathcal{Z}$ is learned by Bayesian approach (\textbf{O2}) using Algorithm~\ref{alg:infer_Z}. The performances are compared with six baselines:
\begin{itemize}[leftmargin=*]
    \item \emph{B1: Task only}--a single objective planning that focuses solely on reaching the goal state and ignores the context;
    \item \emph{B2: LMDP using $\Omega$}--applies the lexicographic MDP formulation~\cite{wray2015multi} assuming the entire state space falls under the highest priority context in $\Omega$;
    \item \emph{B3: Scalarization using $\Omega$}--policy computation using scalarization, with all states mapped to the highest priority context in $\Omega$;
    \item  \emph{B4: LMDP for contexts}--modifies LMDP~\cite{wray2015multi} to plan for multiple contexts in descending order of priority following $\Omega$;
    \item \emph{B5: Adaptive scalarization from Yang et al. (2019)}~\cite{yang2019generalized}--uses a deep neural network (DNN) to learn a policy based on state and scalarization weights as input, producing actions as output;
    \item\emph{B6: Contextual planning w/o resolver}--our approach with $\mathcal{Z}$ given but without using the conflict resolver;
\end{itemize}
\emph{B2} and \emph{B3} are multi-objective planning but with a single ordering over objectives. The following domains are used for evaluation.

\paragraph{Sample Collection using Salp} A salp-inspired underwater robot~\cite{sutherland2010comparative} optimizes the collection and deposition of chemical sample ($o_1$) while minimizing coral damage ($o_2$) and optimizing battery usage by avoiding eddy currents ($o_3$). The robot operates under three contexts: \emph{\{$c_1$: task completion, $c_2$: coral, $c_3$: eddy\}}. Context $c_1$ prioritizes depositing the sample over minimizing coral damage and battery usage: $o_1\!\succ\!o_2\!\succ\!o_3$. Context $c_2$ prioritizes minimizing coral damage, followed by depositing the sample and battery usage: $o_2\!\succ\!o_1\!\succ o_3$. Context $c_3$ prioritizes battery usage, then depositing the sample, and lastly minimizing coral damage: $o_3\!\succ\!o_1\!\succ\!o_2$.

States where the agent carries the samples and is around corals are mapped to the coral context ($c_2$). Locations with eddy currents are assigned to context $c_3$. All remaining states are mapped to $c_1$. The meta-ordering of these contexts is $\Omega \triangleq c_2 \succ c_1 \succ c_3$. 

\paragraph{Semi-Autonomous Taxi} We modify the semi-autonomous driving domain from~\cite{wray2015multi} to consider three objectives and multiple contexts. The agent operates under three objectives: quickly dropping off passengers ($o_1$), maximizing travel on autonomy-enabled roads ($o_2$), and minimizing passenger discomfort by avoiding potholes ($o_3$). The contexts in this environment are: \emph{\{$c_1$: urban transit, $c_2$: self-driving, $c_3$: rough terrain\}}. Context $c_1$ prioritizes passenger drop-off time over comfort and traveling on autonomy-enabled roads, $o_1\!\succ\!o_3\!\succ o_2$. Context $c_2$ prioritizes travel on autonomy-enabled roads followed by drop-off and comfort: $o_2\!\succ\!o_1\!\succ\!o_3$. Context $c_3$ prioritizes minimizing discomfort followed by drop-off and autonomy-enabled travel: $o_3 \succ o_1 \succ o_2$. 
States where the taxi is on autonomy-enabled roads are mapped to context $c_2$, while those going over potholes when a passenger is onboard, are mapped to $c_3$. All other states are assigned to the urban transit context ($c_1$) that prioritizes passenger drop-off. The context ordering for this domain is $\Omega \triangleq c_2 \succ c_1 \succ c_3$.

\paragraph{Package Delivery in Warehouse} The robotic agent in a warehouse~\cite{gao2022two} operates under three objectives: quick package delivery ($o_1$), minimizing damage from slippery tiles ($o_2$), and reducing inconvenience to human workers by avoiding narrow corridors ($o_3$). The contexts in this environment are: \emph{\{$c_1$: normal operation, $c_2$: caution zone, $c_3$: worker zone\}}. Context $c_1$ prioritizes quick package delivery over minimizing slip damage and reducing inconvenience to human workers: $o_1 \succ o_2 \succ o_3$. Context $c_2$ prioritizes minimizing slip damage, followed by quick package delivery and minimizing inconvenience: $o_2 \succ o_1 \succ o_3$. Context $c_3$ prioritizes reducing inconvenience to human workers, followed by optimizing package delivery and minimizing slip damage: $o_3 \succ o_1 \succ o_2$. 
States located areas with slippery tiles are mapped to the caution zone ($c_2$), while those involving navigating through narrow corridors with human workers are mapped to $c_3$. All remaining states are designated to $c_1$. Ordering over contexts for this domain is $\Omega \triangleq c_2 \succ c_1 \succ c_3$.

\paragraph{Expert trajectories for inference} We simulate an expert with complete knowledge of the CLMDP. For each problem instance, 10 trajectories are sampled, with random start states.

\begin{figure}[t]
    \centering
    \begin{subfigure}[t]{\linewidth}
    \centering
    \includegraphics[width=0.72\linewidth]{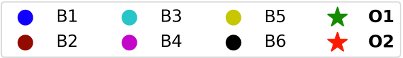}
\includegraphics[width=0.72\linewidth,trim={0.3cm 0 0.2cm 0},clip]{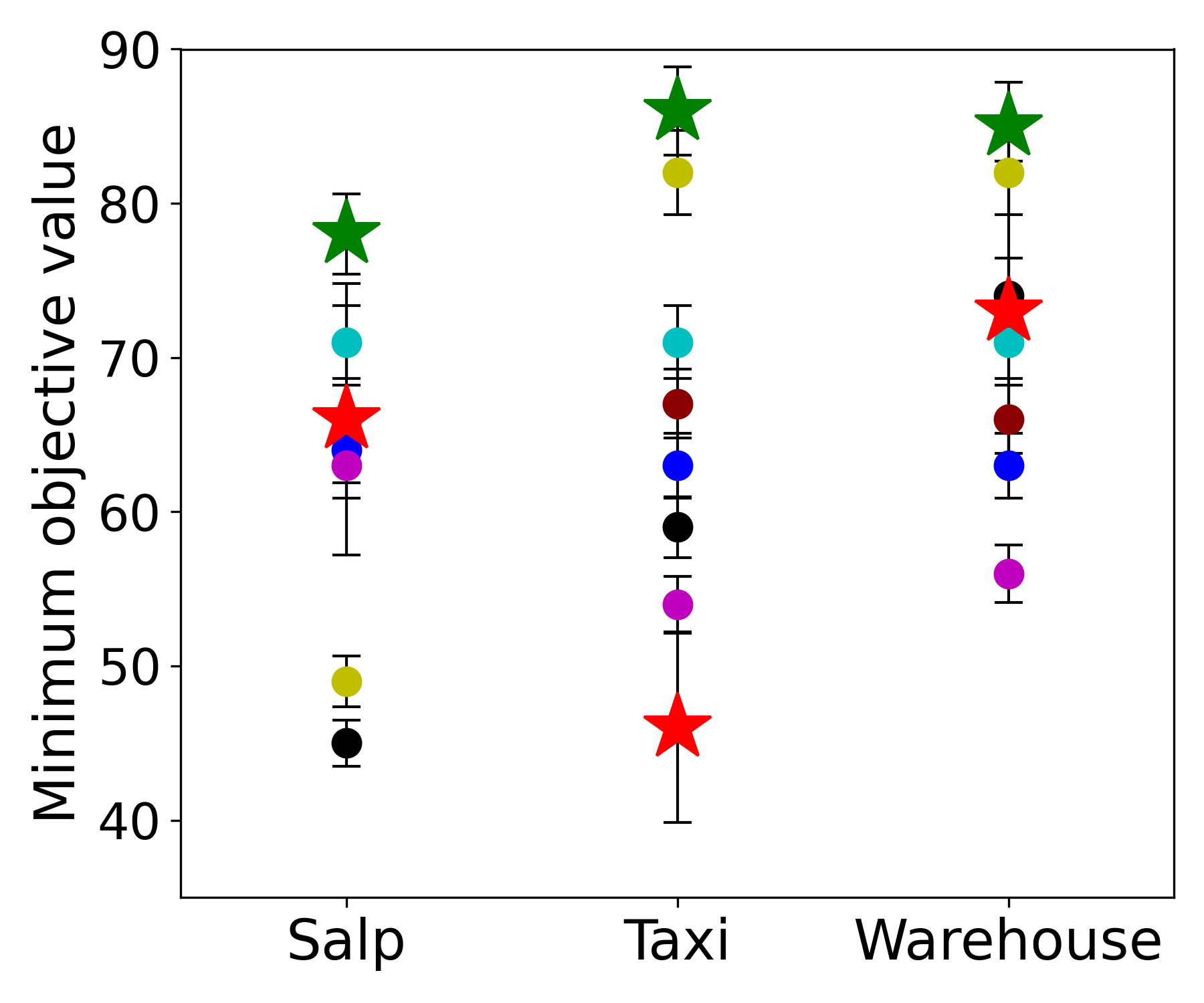}
    \end{subfigure}
    \caption{Minimum objective value from each technique averaged over 100 trials in five instances of each domain.}
\label{fig:min_percentile_consistency}
\end{figure}

\begin{table*}[t]

\caption{Percentage conflicts and task completion for multi-context techniques, averaged over $100$ trials in five instances.}
	\resizebox{7.0in}{!} {
\begin{tabular}{|c|cc|cc|cc|}
\hline
& \multicolumn{2}{c|}{Salp Domain}                  & \multicolumn{2}{c|}{Warehouse Domain}             & \multicolumn{2}{c|}{Taxi Domain}                  \\ \cline{2-7} 
\multirow{-2}{*}{Technique}                & \multicolumn{1}{c|}{\%Conflicts} & \%Goal Reached & \multicolumn{1}{c|}{\%Conflicts} & \%Goal Reached & \multicolumn{1}{c|}{\%Conflicts} & \%Goal Reached \\ \hline
B4: LMDP for contexts & \multicolumn{1}{c|}{$ 40.4\pm 12.1$}        & $39.2\pm 3.5 $          & \multicolumn{1}{c|}{$ 34.2\pm 4.5$}         & $ 28.3\pm 5.2$           & \multicolumn{1}{c|}{$ 13.1\pm 3.2$}         & $ 88.7\pm 9.6$           \\ \hline
B5: Yang et. al, 2019~\cite{yang2019generalized} & \multicolumn{1}{c|}{$28.7 \pm 4.6$}        & $62.4 \pm 7.1$          & \multicolumn{1}{c|}{$11.2 \pm 3.3$}         & $94.1 \pm 5.1$           & \multicolumn{1}{c|}{$4.2 \pm 3.3$}         & $98.6 \pm 13.4$           \\ \hline
B6: Contextual planning w/o resolver                                   & \multicolumn{1}{c|}{$19.8 \pm 5.3$}        & $88.4 \pm 8.1$           & \multicolumn{1}{c|}{$34.2 \pm 7.4$}        & $83 \pm 6.2$             & \multicolumn{1}{c|}{$6.7 \pm 1.5$}         & $95.2 \pm 14.1$           \\ \hline
O1: Contextual planning w/ resolver                                    & \multicolumn{1}{c|}{$0 \pm 0$}           & $100 \pm 0$            & \multicolumn{1}{c|}{$0 \pm 0$}           & $100 \pm 0$            & \multicolumn{1}{c|}{$0 \pm 0$}           & $100 \pm 0$            \\ \hline
O2: Contextual planning w/ resolver \& learned $\mathcal{Z}$                                  & \multicolumn{1}{c|}{$0 \pm 0$ }           & $97.2  \pm 1.72$          & \multicolumn{1}{c|}{$0 \pm 0$ }           & $96.4 \pm 2.24$            & \multicolumn{1}{c|}{$0 \pm 0$ }           & $62.8 \pm 35.65     $       \\ \hline
\end{tabular}}
\label{tab:stats_table}
\end{table*}
\section{Results}
We evaluate the effectiveness of various approaches for contextual multi-objective decision making in terms of: (1) how well a technique can balance the trade-off between different objective values, measured by minimum value across objectives, and (2) validity of the resulting policy, measured in terms of number of conflicts detected and resolved.  
Additionally, we validate our approach using a mobile robot in an indoor setup of warehouse domain (Figure~\ref{fig:trajectory-illustration}). 

\paragraph{Balancing trade-offs between objectives} A heat map showing the performance of each technique, across various objectives and domains, is presented in Figure~\ref{fig:heatmap_all_domains}. Darker shades indicate higher (better) objective values. The approach from Yang et. al, (2019)~\cite{yang2019generalized} occasionally surpasses our approach in individual objectives, but it does not perform consistently well across all domains. Our approach performs consistently well, across objectives and domains, when the state-context mapping $\mathcal{Z}$ is known (O1). 

While some baseline techniques perform well in certain objectives, even outperforming our approach, they tend to score lower on others. In a multi-objective setting, the performance of a technique must be evaluated based on how well it can perform \emph{across all objectives}. We evaluate this performance trade-off across objectives via a percentile comparison showing the minimum objective value. 
We compare the lowest-performing objective value, across all objectives, normalized against the maximum achievable reward in that objective. This value is averaged over 100 trials,  across five instances in each domain.
Figure~\ref{fig:min_percentile_consistency} shows the percentile comparison showing the minimum objective value for each technique across all domains. Note that our approach consistently achieves the highest minimum objective value. Additionally, we also evaluate our second approach, contextual planning with resolver with a learned $\mathcal{Z}$ (O2), where the state-context mapping is inferred from expert trajectories (as described in Section~\ref{sec:learning_Z}). While this approach shows weaker performance, it should be noted that this is due to inaccuracies in predicting the context with limited expert trajectories, affecting both the objective ordering and the associated reward function. The performance with learned $\mathcal{Z}$ will likely improve with the availability of additional data.

\paragraph{Resolving conflicts}  
Table~\ref{tab:stats_table} provides a comparative analysis of the percentage of conflicts (cycles) and task completions for techniques that plan using \emph{context} information. 
Our approach (O1) successfully avoids all conflicts and achieves 100\% task completion in every trial across all domains. This consistent performance highlights the effectiveness of our conflict detection and resolution mechanisms (Alg.~\ref{alg:policy_conflict_checker} and~\ref{alg:policy_conflict_resolver}). 
Our approach  with learned $\mathcal{Z}$ (O2) does not always reach the goal, even in the absence of conflicts in the policy. This is due to \emph{incorrect attribution of contexts}, which leads to suboptimal policies that hinder the agent's ability to reach the goal state. When context partitions are incorrectly assigned, 
Assumption~\ref{assumption:policy-exist} does not hold. By Defn.~\ref{def:conflict}, a policy is said to have a conflict only if the goal reachability is zero. Some policies may have a non-zero (but less than one) probability of reaching the goal. Thus, some conflict-free policies may not always reach the goal during execution.

\begin{figure}[t]
    \centering
   \includegraphics[width=0.94\linewidth]{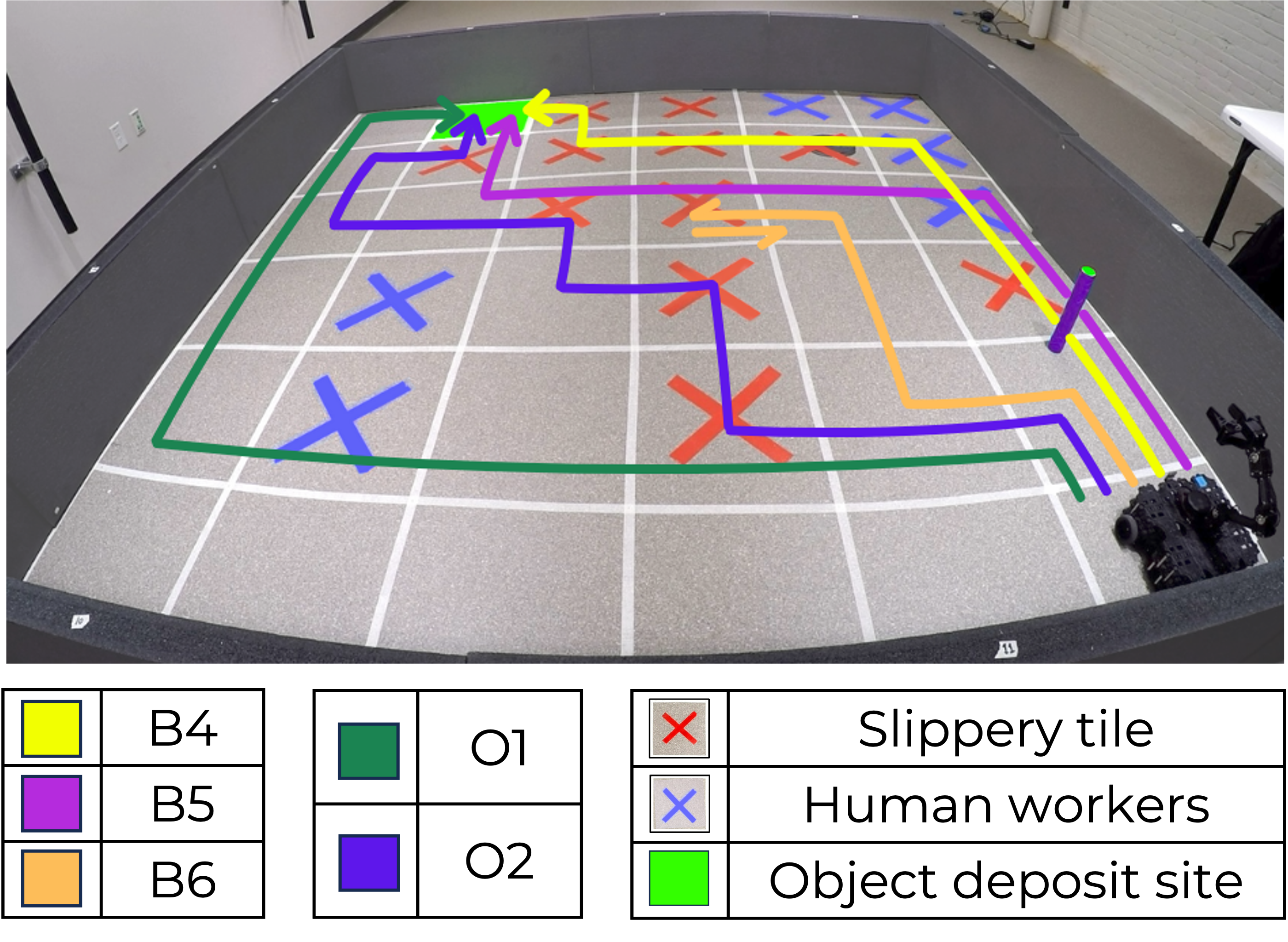}
    \caption{Comparison of the paths taken by our warehouse agent during delivery task in our indoor setup.}
    \label{fig:trajectory-illustration}
\end{figure}

\begin{table}[t]
\caption{Normalized objective values from hardware experiments in the warehouse domain, averaged over five trials.}
\begin{tabular}{|c|c|c|c|}
\hline
Technique                                                                   & $o_1$ & $o_2$ & $o_3$ \\ \hline \hline
B4: LMDP for contexts                                                                   & 0.84   & 0.72 & 0.94  \\ \hline
B5: Yang et. al, 2019 [37]                                                 & 0.48   & 0.43  & 0.83  \\ \hline
\begin{tabular}[c]{@{}c@{}}B6: Contextual planning w/o resolver\end{tabular} & 0.0    & 0.0    & 0.0    \\ \hline
\begin{tabular}[c]{@{}c@{}}O1: Contextual planning w/ resolver\end{tabular}  & 0.98  & 0.90  & 0.93   \\ \hline
\begin{tabular}[c]{@{}c@{}}O2: Contextual planning w/ resolver\\ \& learned $\mathcal{Z}$\end{tabular}  & 0.73  & 0.88  & 0.99   \\ \hline
\end{tabular}
\label{tab:hardware_results}
\end{table}

\paragraph{Evaluation using a mobile robot}  
We conduct a series of experiments using a TurtleBot in an indoor warehouse domain setup. The robot autonomously collects and delivers an object using a LiDAR and a map of the area for active localization so as to determine its state and execute actions based on its computed policy. Figure~\ref{fig:trajectory-illustration} shows the trajectories corresponding to different techniques. Table~\ref{tab:hardware_results} shows the normalized objective values. Our contextual planning approach performs well across all objectives. Without conflict resolver, the robot does not reach the goal, emphasizing the importance of conflict resolver module. Overall, the results demonstrate the method's applicability in real-world robotic tasks with context-specific preferences over multiple objectives.

\section{Conclusion}
This paper presents contextual lexicographic MDP, a framework for multi-objective planning in scenarios where the relative preferences between objectives and the associated reward functions are determined by the context. We also present an algorithm to solve the CLMDP, by combining independently computed policies from different contexts into a global policy that is context-aware and cycle-free. We also analyze the algorithm's correctness and complexity. Empirical evaluations in simulation and using a mobile robot demonstrate the effectiveness of our approach in balancing the performance trade-offs across objectives, under multiple contexts, and the feasibility of applying our approach in the real world. 

In the future, we aim to extend our algorithm to support non-stationary state-context mapping. This work addresses the technical challenges in planning with context-specific objective orderings, assuming this information is known or can be inferred from expert data. How to define a context for a problem and how to define a mapping between objective orderings and contexts extend beyond technical considerations and into ethical challenges. Prioritizing one objective over another are often based on subjective input from domain experts and we acknowledge that these decisions carry significant ethical implications that cannot be overlooked.                                                         

\begin{acks}
This work was supported in part by ONR grant N00014-23-1-2171. We thank Akshaya Agarwal for their help with robot experiments.
\end{acks}

\newpage
\bibliographystyle{ACM-Reference-Format} 
\balance
\bibliography{references}
\end{document}